\PassOptionsToPackage{dvipsnames}{xcolor}
\documentclass{applemlr}

\usepackage[dvipsnames]{xcolor}
\hypersetup{pdftitle={Weighted State Souping for Efficient Long-Context Retrieval},
  pdfauthor={William L. Tong, Aryo Lotfi, Emmanuel Abbe, Josh Susskind, Cengiz Pehlevan, Eran Malach}}
\usepackage{graphicx}
\graphicspath{{./}}

\ifPDFTeX
  \usepackage[utf8]{inputenc} %
  \usepackage[T1]{fontenc}    %
\fi
\usepackage{url}            %
\usepackage{booktabs}       %
\usepackage{amsfonts}       %
\usepackage{nicefrac}       %
\usepackage{microtype}      %

\usepackage{amsmath}
\usepackage{amssymb}
\usepackage{mathtools}
\usepackage{amsthm}
\usepackage{bbm}

\newcommand{\R}{\mathbb{R}}
\newcommand{\vr}[1]{\mathbf{#1}}

\newtheorem{proposition}{Proposition}
\newtheorem{remark}[proposition]{Remark}

\title{Weighted State Souping \\for Efficient Long-Context Retrieval}

\author[2,*]{William L. Tong}
\author[1]{Aryo Lotfi}
\author[1]{Emmanuel Abbe}
\author[1]{Josh Susskind}
\author[2]{Cengiz Pehlevan}
\author[1]{Eran Malach}

\affiliation[1]{Apple}
\affiliation[2]{Harvard University}
\contribution[*]{Work done while at Apple}

\abstract{
State space models (SSMs) combine fast inference and fixed state size with competitive language-modeling performance. Recently, \textit{state souping} has been proposed as a procedure whereby an SSM independently encodes chunks of tokens like documents, tool definitions, or memory records. The resulting state encodings may be cached and recombined to cheaply construct a context for downstream queries. However, both state souping and vanilla SSMs deteriorate on long, crowded contexts, limiting their utility for long-horizon domains. To address this issue, we propose \textit{weighted state souping}, a simple method that learns query-dependent weights for composing cached states. The resulting models improve long-context performance substantially across both single- and multi-hop document question answering, needle-in-a-haystack retrieval, variable tracking, and tool use. Our method can be further adapted into an efficient sparse variant, yielding near-constant encoding cost with respect to context length while preserving (and sometimes \textit{improving}) performance. We also present simple theoretical arguments that illustrate why souping works and when it may fail, and discuss informally the advantages of weighted souping. Weighted state souping thus combines reusable context representations with selective retrieval, offering an efficient approach to language modeling over large collections of documents, tools, and other forms of structured context that increasingly dominate modern agentic workloads.
}

\metadata[Correspondence]{\sffamily William L. Tong: \url{wtong@g.harvard.edu}; Eran Malach: \url{e_malach@apple.com}}
\date{\sffamily October 18, 2026}

\begin{document}

\maketitle

\section{Introduction}
Many applications of large language models (LLMs) require long-lived context organized into discrete chunks of tokens: tool definitions for an agent \citep{schick2023toolformer,patil2025bfcl}, documents for question answering or retrieval-augmented generation \citep{rajpurkar2016squad,yang2018hotpotqa,trivedi2022musique,lewis2020rag,guu2020latent_rag}, and memory records or chat histories that persist across interactions \citep{packer2023memgpt,wu2024longmemeval,salemi2024lamp_personalization}. A useful memory architecture must retrieve relevant chunks from an increasingly crowded store without repeatedly processing the entire history.

In these settings, common practice concatenates the chunks of tokens and processes them as a single flat sequence. With the rising popularity of state space models (SSMs) and related fixed-state architectures \citep{gu2024mamba,dao2024mamba2,katharopoulos2020linear_transformers,schlag2021fast_weights_delta,sun2024ttt}, an alternative approach becomes possible. Rather than perform a single autoregressive pass on a concatenated sequence, each chunk can be independently encoded into a state. The resulting collection of states are then point-wise summed to initialize the model before it processes a query (Figure \ref{fig:soup_illustration}). By analogy to model souping \citep{wortsman2022modelsoups}, this procedure is called \textit{state souping} \citep{piualro2024state_soup,jafari2025surprising_soup,liu2025picaso_soup}: the quantities combined are input-dependent recurrent states, rather than model parameters.

State souping has an obvious efficiency advantage. State encodings can be computed in parallel and, for slowly changing chunks like memory records or tool definitions, cached and reused. This ability to cache representations makes souping naturally suited for handling the structured context that increasingly dominates modern agentic workloads: inference can depend on selected state encodings without repopulating the context with large memory artifacts, tool definitions, and documents. However, finite-state SSMs struggle to copy or retrieve from long contexts \citep{jelassi2024repeat_after_me,waleffe2024empirical_mamba,arora2024zoology}, and state souping performance similarly decays as more documents are added \citep{jafari2025surprising_soup}. This loss of selectivity limits souping precisely where persistent memory is most useful.

In this work, we introduce \textit{weighted state souping}, which learns attention-like weights over state encodings. Across document QA, needle-in-a-haystack retrieval, variable tracking, and tool use, weighting substantially improves robustness as contexts grow beyond training lengths. Our method can be further adapted into a sparse variant that permits near-constant encoding cost while preserving --- and sometimes \textit{improving} --- performance. Our main contributions are:
\begin{itemize}
    \item a novel weighted state souping procedure that permits robust selectivity, strong generalization on long contexts, and substantial efficiency gains over both Transformers and traditional SSMs,
    \item a comprehensive evaluation over diverse benchmarks including multi-document question answering, needle-in-a-haystack retrieval, variable tracking, and agentic tool use,
    \item a simple theoretical account of how souping successfully retrieves and composes across chunks, and its potential weaknesses compared against sequential models. We also discuss informally the advantages of weighted state souping.
\end{itemize}

\section{Related work}
State souping is a relatively recent method. \citet{piualro2024state_soup} introduced the term, and explored linear combinations of recurrent states for in-context skill transfer. \citet{jafari2025surprising_soup} study independently encoded document states, showing that finetuning enables simple pooling operators to support downstream QA and retrieval. PICASO \citep{liu2025picaso_soup} derives state-composition rules similar to souping for permutation-invariant inputs. We build on these works by showing how query-dependent weighting permits both significantly stronger length generalization and substantial efficiency wins.

State souping can be seen as an approximate form of state caching, a procedure whereby states generated from common prefixes can be cached and reused exactly for downstream requests. In Transformers, the procedure is more often called KV caching \citep{shazeer2019kv_cache,kwon2023vllm}. In SSMs, state caching has been explored particularly in the context of hybrid models, where common  KV caching methods for attention blocks may not transfer directly to SSM layers. In particular, \citet{pan2025marconi} explore SSM-specific cache eviction policies to maintain an efficient prefix cache for SSMs. \citet{shirokikh2026sparse_prefix_caching} explore a complementary idea in which cache checkpoints are saved along a sequence; future generation resumes from the deepest checkpoint that matches the current prefix. \citet{gopal2026structured_memory} investigate state caching in pure SSMs for retrieval augmented generation, and highlights its efficiency for edge-computing applications.

Souping is also related to state tuning, the process of learning an initial state for a recurrent model. This process has been studied in the context of parameter-efficient fine tuning (PEFT) for state space models \citep{galim2024ssm_peft,kang2025state_offset_tuning,yoshimura2025mambapeft}. \citet{kang2025state_offset_tuning} in particular find that \textit{state offset tuning}, in which a learned offset is added to every recurrent state, outperforms initial state tuning and other PEFT methods like LoRA. State tuning has also been studied in the RWKV community \citep{xiao2025state_tuning}, and historically in early RNNs \citep{forcada1995learning_init_state}.

Our weighted approach to state souping is similar to methods that attempt to introduce attention-like features to state space models, interpolating between Transformers and SSMs. In particular, \citet{akhauri2024attamba} introduce a method in which SSMs produce compressed key-value representations before applying attention. Similarly, Infini-attention \citep{munkhdalai2024_infini_att} combines local attention with a compressive recurrent memory. Memory Caching \citep{behrouz2026memory_caching_rnns} explores storing and combining recurrent memory checkpoints through attention-like weights. Our approach focuses on structured contexts that permit state souping, and our weighting does not retain token-by-token attention to cached states during generation, ultimately yielding both substantial efficiency benefits and performance gains.

Methods analogous to state souping have also been investigated in Transformers. Notably, \citet{eyuboglu2026cartridges} introduce a method for directly training KV caches to encode documents. The resulting ``cartridges'' can be provided in-context to approximate the behavior of the model had it been prompted with the original documents, similar to how SSM state encodings are used in souping. However, cartridges require offline tuning specific to each document, and multiple cartridges do not compose in-context without significant additional training \citep{hardalov2026cartridges_at_scale}. Our weighted state souping approach is also similar in motivation to block attention architectures that use block representations to select relevant blocks for token-level attention \citep{mohtashami2023landmark_att,lu2026moba} or use query-dependent scores to load KV cache blocks during decoding \citep{tang2024quest}.

\section{Weighted State Souping} \label{sec:weighted_state_souping}

\begin{figure}
    \vspace{-1em}
    \centering
    \includegraphics[width=0.9\linewidth]{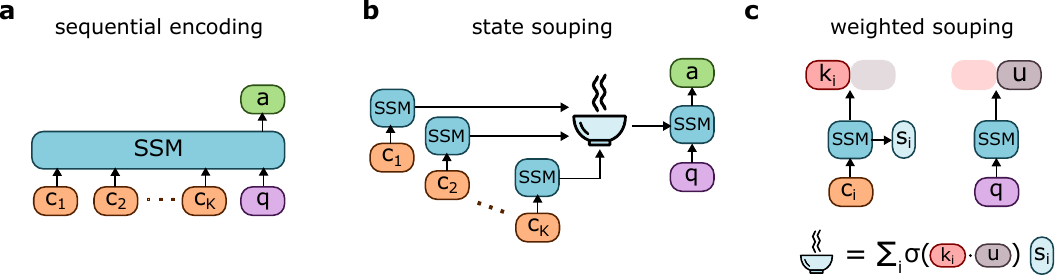}
    \caption{\textbf{Sequential encoding versus state souping.} \textbf{(a)} A typical \textcolor{NavyBlue}{SSM} operates by concatenating \textcolor{orange}{input chunks} together with the \textcolor{violet}{user query}, and performs a single, long autoregressive pass to produce the \textcolor{Green}{answer}. \textbf{(b)} A state souping \textcolor{NavyBlue}{SSM} encodes each \textcolor{orange}{input chunk} independently, then soups the resulting state encodings into a single encoding used to initialize the model as it processes the \textcolor{violet}{user query}. \textbf{(c)} In our weighted state souping method, we augment the \textcolor{NavyBlue}{SSM} with an additional linear readout that produces \textcolor{BrickRed}{key vectors} from the \textcolor{orange}{input chunks} and a \textcolor{gray}{query vector} from the \textcolor{violet}{user query}. The soup is constructed as a weighted sum with weighting function $\sigma$.}
    \label{fig:soup_illustration}
    \vspace{-0.1in}
\end{figure}

State souping operates in the following way. Let an input contain $k$ token chunks $\vr{c}_1,\ldots,\vr{c}_k$ and query $\vr{q}$. The SSM independently encodes each chunk into its final state $\vr{s}_i$, where the state encompasses the collection of recurrent states across all layers. We then perform a point-wise average over the $k$ encodings to make the ``soup'' $\vr{s}_0 = \frac{1}{k}\sum_{i=1}^k\vr{s}_i$. $\vr{s}_0$ becomes the initialization for the SSM as it processes the user query (Figure \ref{fig:soup_illustration}). We use the same SSM for encoding and generation, but these models can be different. Souping is not an inference-time only intervention, and requires either learning from scratch or additional finetuning on a pretrained SSM to teach the model to used this composed initialization \citep{jafari2025surprising_soup}.

Crucially, the only channel through which the input chunks influence the query is the souped initialization. Since SSMs perform simple linear state updates, a soup constructed through an average of states can be viewed as a naive approximation of the state construct from a concatenated, sequential input. Such a \textit{mean soup} is the approach used in prior work \citep{piualro2024state_soup,jafari2025surprising_soup}, and performs comparably with a sequential SSM. However, both mean souping and sequential SSMs struggle when operating over long contexts containing many input chunks \citep{jafari2025surprising_soup,chen2024stuffed_mamba,arora2024zoology}. As the number of chunks increases, a mean soup or finite SSM state becomes increasingly noisy, degrading performance. 

Query-dependent weighting may instead concentrate the souped state on relevant chunks, the inspiration for \textit{weighted state souping}. Our method is related to other schemes that weight input tokens in sequential SSMs \citep{akhauri2024attamba,behrouz2026memory_caching_rnns}, though our focus is on a souping architecture for chunked inputs. To compute these weightings, we add a learnable linear readout that maps the final activation to a key and query vector. Each chunk produces key $\vr{k}_i$. Starting from a zero-initialized state, a pass over the user query produces $\vr{u}$. We then compute the souped initialization as $\vr{s}_0 = \sum_{i=1}^k \sigma(\vr{k}_i \cdot \vr{u}) \, \vr{s}_i$, with weighting function $\sigma (\cdot)$. The linear readout and base model are jointly finetuned. See Figure \ref{fig:soup_illustration}c for an illustration of this \footnote{The weighting function $\sigma (\cdot)$ can include a normalization term, making it effectively a function of all keys. For simplicity, we omit the addition of arguments that occur only in the normalization.}. Unlike a Transformer's token-level attention, this procedure assigns one scalar weight per chunk, shared across all layers and heads and computed once from the original user query. The weights remain fixed during query processing and generation.

For $\sigma$, we use the softmax function with cosine-normalized scores (see Appendix \ref{sec:impl_details} for details). We also explore \textit{sparse souping}, an efficient variant in which we mask at inference time all but the top-$N_{top}$ logits before softmax, retaining normalized weights while avoiding retrieving zero-weight states. Masked states do not need to be retrieved, providing significant efficiency advantages for sparse souping (Section \ref{sec:computation}). Other reasonable choices of weighting functions include sigmoidal and linear weights, though we find that softmax performs best in our experiments. 

\section{Experiments} \label{sec:experiments}
We use a 1.5B Mamba-2 SSM \citep{dao2024mamba2} and compare softmax-weighted, sparse, and mean souping against sequential Mamba-2 and a matched Transformer with rotary positional encodings \citep{su2021rope}.

We evaluate our models across a diverse range of tasks. For document QA, we use SQuAD \citep{rajpurkar2016squad} for single-hop; HotpotQA \citep{yang2018hotpotqa} and MuSiQue \citep{trivedi2022musique} for multi-hop. For NIAH and variable tracking, we use RULER \citep{hsieh2024ruler}. For tool use, we use the Berkeley Function Calling Leaderboard (BFCL) \citep{patil2025bfcl}. To probe long-context performance on document QA and BFCL, we augment the input with additional distractors and plot performance for varying pool size $k$. If $k$ is smaller than the number of gold documents or gold tool definitions (relevant only for the $k=1$ datapoints), then the input consists only of gold samples. This suite tests selective retrieval and reasoning as the available context grows, crowding the relevant chunks with increasingly more distractors. Our models are finetuned to each task, using either the task-provided official training split (for document QA), a specialized post-training corpus (for tool use), or a disjoint subset of examples (for RULER). See Appendix \ref{sec:impl_details} for additional details on training, evaluation, and our benchmarks.

\subsection{Benchmark results}

\begin{figure}
    \centering
    \includegraphics[width=0.85\linewidth]{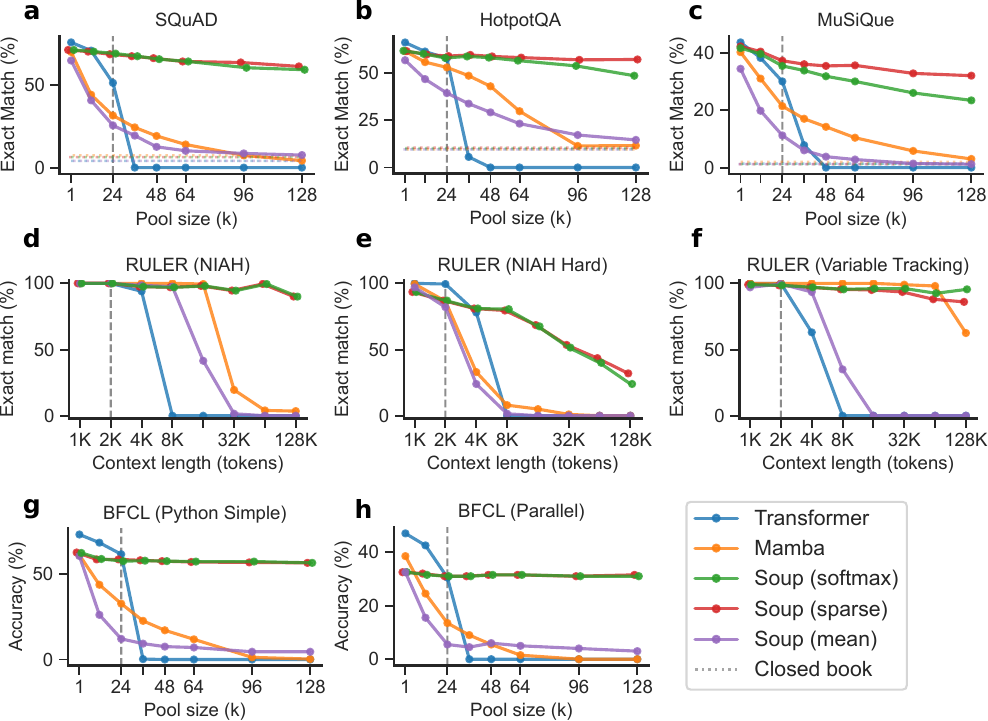}
    \caption{\textbf{Weighted state souping benchmark results.} Vertical gray dashed lines indicate the max pool size or context length observed during training. \textbf{(a - c)} Single-hop (SQuAD) and multi-hop (HotpotQA, MuSiQue) benchmarks. We plot accuracy as an exact match between the gold answer span and the model prediction. Dotted lines represent closed book accuracy, where no documents are prompted in-context. Sparse souping keeps top $N_{top}=4$ chunks. \textbf{(d,e)} NIAH subsets of RULER. In the easy version, a single needle consisting of a random number is hidden among a natural language haystack of variable size. In the hard version, the haystack consists entirely of needle-like distractors, requiring precise retrieval. Sparse souping keeps top $N_{top}=20$ chunks. \textbf{(e)} Variable tracking subset of RULER. The input consists of random variable assignments embedded in random noise sentences, and the model must infer the correct value of a randomly chosen variable. Sparse souping keeps top $N_{top}=20$ chunks. \textbf{(f)} Simple Python subset of BFCL. The model is prompted with a list of Python functions and descriptions, and must execute a correct function call given a user query. Sparse souping keeps top $N_{top}=4$ chunks. \textbf{(g)} Parallel subset of BFCL. The setup is the same as in \textit{(f)}, but the model needs to execute the function multiple times.}
    \label{fig:results}
\end{figure}

Figure \ref{fig:results} summarizes the main results; Appendix  \ref{sec:additional_figs} reports on additional metrics and task subsets.

In general, weighted state souping performs strongly even at large pool size $k$. While mean souping and sequential Mamba tend to deteriorate as the number of chunks increases, weighted souping matches their performance at small $k$, and persists in near-constant performance as $k$ increases to several times beyond the max training size. This phenomenon persists across a diversity of tasks and benchmarks, including in multi-hop settings like HotpotQA, MuSiQue, and RULER variable tracking where the model must link information across multiple chunks. Similar to other works \citep{malach2026infinity, zhou2024transformers}, we observe that Transformers do not extrapolate well to sequences longer than the training data, though their in-distribution performance tends to be strong.

Interestingly, sparse souping matches or improves on full softmax souping, particularly on multi-hop QA (HotpotQA, MuSiQue). Softmax mass may diffuse across many documents as $k$ increases. Sparse souping instead may beneficially concentrate attention in the most relevant indices. An analogous benefit has been observed in sparse attention for Transformers \citep{velivckovic2024softmax_not_enough,vasylenko2025long_context_sparse,zhang2024attention_dilution}. One weakness of sparse souping is the need to specify the number of chunks $N_{top}$ to keep. If $N_{top}$ is insufficiently large, critical chunks may be dropped and the model's performance suffers (Figure \ref{fig:ruler_results}f). Designing a weighting function with adaptive sparsity represents an interesting avenue for further exploration.

While the performance of weighted souping is generally strong, it underperforms at small $k$ in the \textit{parallel multiple} subset of BFCL, despite strong performance on \textit{python simple} and \textit{parallel} (Figure \ref{fig:bfcl_results}). Whether this reflects an optimization issue or a fundamental limitation remains open.

\subsection{Computational efficiency} \label{sec:computation}

\begin{figure}
    \centering
    \includegraphics[width=0.85\linewidth]{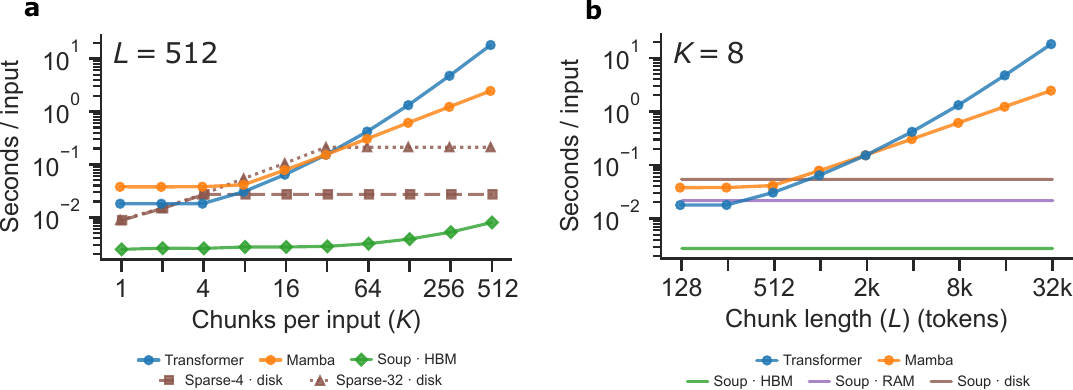}
    \caption{\textbf{Computational cost of souping.} We measure the time it takes to encode a set of input chunks on an AWS H100 node. For the souping models, we include the time it takes to encode a 128 token query and compute the softmax weights. The state encodings are retrieved from either GPU high bandwidth memory (HBM), system RAM, or system disk. We plot a subset of timings here, with the full results in Figure \ref{fig:full_compute_time}. We include two sparse souping variants, one that keeps the top $N_{top}=4$ documents and another that keeps the top $N_{top}=32$. (\textbf{a}) We hold the chunk length fixed at $L=512$ tokens and increase the total number of chunks. Encoding time therefore grows linearly with increasing chunks for all models. For the HBM model, at low $K$ encoding the query dominates the encoding time, resulting in a flat trend. (\textbf{b}) We hold the number of chunks fixed at $K=8$ and increase the length of each chunk. Since the number of chunks is constant, souping retrieves the same number of chunks for each query, resulting in a flat trend across all $L$ where the height is controlled by retrieval speed.}
    \label{fig:compute}
\end{figure}

The primary computational advantage of souping is in its ability to cheaply retrieve a state encoding from memory, rather than autoregressively re-encoding every input. Since inputs like memory artifacts and tool definitions change slowly, their state encodings may be routinely cached for future retrieval. Hence, a souped model may be considerably faster than non-souped alternatives, particularly as the context length grows. Figure \ref{fig:compute} compares the measured time to encode an input between souping, Mamba, and Transformer models. See Figure \ref{fig:full_compute_time} for additional timings.

There are a number of further considerations for souping. 1) Weighted souping must also factor in the time to perform an initial pass over the user query to generate a query vector, and the time it takes to compute the weightings. For a fixed key dimension, scoring and normalizing the weights costs $O(k)$ per query. Once the soup is constructed, the cost per generated token is constant with respect to $k$, since generation does not access the cached chunks again. In general, user queries average under 100 tokens \citep{zheng2024lmsys} and computing softmax weights is extremely fast, so the total added cost of computing the weights is typically negligible and amortized over the number of input chunks. 2) The retrieval time is heavily influenced by the specific memory holding the cached state encodings. GPU HBM is significantly faster than system RAM, which in turn is faster than system disk. Depending on where the encodings live, retrieval may in the worst case be slower than encoding from scratch. 3) Sparse souping improves efficiency further, since chunks assigned zero weight do not need to be retrieved. Hence, if the number of nonzero weights is constant with respect to the context length (which is true in our setting), then the encoding performance of sparse souping is effectively constant as context length grows\footnote{Technically, the number of weights increases with context size as well. However, since computing weights is considerably faster than retrieval, the context size needs to be impractically immense before the time it takes to compute the weights becomes noticeable.}.

\section{Theory: retrieval and composition} \label{sec:theory}
The experiments in Section~\ref{sec:experiments} provide the main evidence for the effectiveness of weighted state souping. To complement these empirics, we present a simple theoretical analysis that begins to answer the questions of why souping works, and what potential strengths and weaknesses it may have compared with sequential models. Our arguments are primarily expressivity-based: we provide constructions showing that both souping and sequential models can implement an induction-head retrieval task representative of our empirical benchmarks. We also discuss possible weaknesses that independent chunk encoding introduces, namely a sensitivity to chunk boundaries and the inability to encode cross-chunk representations. We conclude with a discussion of the benefits weighted souping has for retrieval over long, crowded contexts. Our analysis is by no means a complete theoretical account of state souping. Instead, we introduce simple intuitions that build towards a principled understanding of the method.

\subsection{Task}
We study a key-value version of the induction-head retrieval task \citep{olsson2022induction_heads} and its multi-hop extension \citep{sanford2024k_hop}.

Let $[N]=\{1,\ldots,N\}$, and introduce disjoint key and value vocabularies $\mathcal K=\{k_a:a\in[N]\}$ and $\mathcal V=\{v_a:a\in[N]\}$. An input context defines a partial function $f:[N] \rightarrow [N]$ by listing pairs $(k_a,v_{f(a)})$, with each key appearing at most once. Given a query $k_q$, the
$n$-hop induction target is
\[
    y_n(f,q)=v_{f^{n}(q)},
\]
whenever $f(q),\ldots,f^n(q)$ are defined. Thus one hop retrieves the value paired with $k_q$; each subsequent hop uses the index of the retrieved value as the next key. The input context is presented to the model as concatenation of all pairs followed by the query,
\[k_{a} \,, v_{f(a)} \,, k_{b} \,, v_{f(b)} \,, \ldots \,, k_{q} \,.\]
A sequential model reads this sequence in one pass. A souping model encodes each pair independently. A chain of length $n$ is a sequence of distinct indices $a_0,\ldots,a_n$ with $f(a_i)=a_{i+1}$. Multiple chains and irrelevant pairs may appear in the context. Pair order is arbitrary except in the explicitly ordered setting below.

\subsection{Model} \label{sec:model}
We consider a stripped-down linear attention SSM with Mamba-like history decay and augmented residuals. Let $z_1,\ldots,z_T$ be the tokens and let $\vr e(z_t)\in\R^d$ be orthonormal vocabulary embeddings, with $d\geq2N$. Initialize the residual stream as $\vr r_t^{(0)}=[\vr e(z_t);\vr0]\in\R^{2d}$. At layer $\ell\in[D]$, a causal width-two convolution copies the previous and current residual vectors into separate coordinates:
\begin{equation}
    \vr x_t^{(\ell)}=
    \begin{bmatrix}I_{2d}\\0\end{bmatrix}\vr r_{t-1}^{(\ell-1)}+
    \begin{bmatrix}0\\I_{2d}\end{bmatrix}\vr r_t^{(\ell-1)}
    =\begin{bmatrix}\vr r_{t-1}^{(\ell-1)}\\\vr r_t^{(\ell-1)}\end{bmatrix}\in\R^{4d}.
    \label{eq:theory-convolution}
\end{equation}
This is the copying role of the convolution identified by \citet{huang2025mamba_input_sel}. We then perform the SSM state update as
\begin{align}
H_t^{(\ell)}&=\alpha H_{t-1}^{(\ell)}+
(W_k^{(\ell)}\vr x_t^{(\ell)})(W_v^{(\ell)}\vr x_t^{(\ell)})^\top,
&\vr h_t^{(\ell)}&=(H_t^{(\ell)})^\top W_q^{(\ell)}\vr x_t^{(\ell)},
\label{eq:linear-attention-state}
\end{align}
where $H_t^{(\ell)}\in\R^{d\times d}$ is the SSM state; $W_k^{(\ell)},W_v^{(\ell)},W_q^{(\ell)}\in\R^{d\times4d}$ are the key, value, and query projections; and $\alpha\in(0,1]$ is the history decay. We omit bias terms.

The residuals are updated as
\begin{align*}
\vr r_t^{(\ell)}&=\vr r_t^{(\ell-1)}+W_o\vr h_t^{(\ell)},
&W_o&=\begin{bmatrix}0\\I_d\end{bmatrix}.
\end{align*}
In this way, $\vr r_t^{(\ell)}=[\vr e(z_t);\vr u_t^{(\ell)}]$ preserves the token embedding in the first block and accumulates updates in the second block. 

The final output logits produced by the model are $U\vr u_T^{(D)}$, where $U\in\R^{N\times d}$ is the readout matrix. To extract a concrete next-token prediction, we take the argmax over these logits.

Sequential encoding starts from zero-initialized recurrent states and a zero convolution buffer. For souping, each chunk $c_i$ is encoded independently from zero at every layer, giving states $H_i^{(\ell)}$. The query then starts from the weighted, souped initialization $H_{\mathrm{init}}^{(\ell)}=\sum_{i=1}^K w_iH_i^{(\ell)}$, with its convolution buffer reset to zero. The weights $w_i$ are shared across layers and fixed during query processing. Unless stated otherwise, chunk encoding and weighting have no access to the chunks' original positions: permuting the chunks permutes their weights and leaves the soup unchanged. This assumption is accurate of mean and content-based softmax souping. Our discussion below focuses on expressivity rather than strict learnability, and assumes exact arithmetic and suitably chosen positive weights.

\begin{remark} \label{remark:linear_states_soup}
Linear states permit souping. For a single layer with matching convolution resets, sequential encoding equals a weighted sum of independently encoded chunk states.
\end{remark}
There is a simple algebraic reason that states from this model can be souped. Split a context into chunks $c_1,\ldots,c_K$ of lengths $m_1,\ldots,m_K$, and let $H(c_i)$ be the final state obtained by encoding $c_i$ from zero. Let $H_{\mathrm{seq}}(c_1\Vert\cdots\Vert c_K)$ be the final state obtained by encoding the concatenated sequence of chunks. If we reset the convolution at the same boundaries in both schemes, then unrolling~\eqref{eq:linear-attention-state} gives
\begin{equation}
    H_{\mathrm{seq}}(c_1\Vert\cdots\Vert c_K)
      =\sum_{i=1}^K\alpha^{\sum_{j>i}m_j}H(c_i).
    \label{eq:sequential-is-soup}
\end{equation}
Thus a sequential linear state is itself a weighted soup of that layer's chunk states, an identity also observed in prior state-composition work \citep{liu2025picaso_soup}. With $\alpha=1$ and a one layer model, sequential and souping are functionally equivalent.

However, there are a few important differences to keep in mind. On one hand, learned state souping replaces the fixed recency weights $\alpha^{\sum_{j>i}m_j}$ in~\eqref{eq:sequential-is-soup} with content-dependent weights, permitting selection based on relevance rather than position. On the other hand, independently encoded chunks cannot incorporate information from other chunks during encoding. Particularly when a model has two or more layers, its sequential upper-layer representations may therefore differ from its independently encoded representations, potentially impacting efficiency when cross-chunk representations would be advantageous (see Remark~\ref{remark:capacity} below).

\subsection{One-hop induction}
We begin by considering the one-hop induction-head task, reflecting the setting in \citet{olsson2022induction_heads}. We observe that both a sequential and souping model can solve the task, but chunk boundaries that separate the task's key-value pairs can make souping fail.

\begin{proposition} \label{prop:one-hop}
Suppose that every key-value pair is contained in one chunk, each key occurs at most once, and the queried pair's chunk receives positive soup weight. A single layer model with sufficient width (Section \ref{sec:model}) solves one-hop induction for both sequential encoding and weighted state souping.
\end{proposition}

See Section \ref{sec:proof_one_hop} for the construction proving this statement. Proposition \ref{prop:one-hop} demonstrates that even a single-layer SSM is sufficiently expressive to perform one-hop induction. Together with our earlier discussion in Remark \ref{remark:linear_states_soup}, we observe that both souping and sequential models can readily perform in-context retrieval. However, souping is sensitive to chunk boundaries, as we establish in the next proposition.

\begin{proposition} \label{prop:split-pair}
Suppose that chunk boundaries split every key-value pair (i.e. every token is encoded as a separate chunk). Then no souping model solves all one-hop induction instances.
\end{proposition}
\begin{proof}
Choose distinct indices $a,b,u,v$. Consider the two token sequences
\begin{align*}
    C  &= (k_a,v_u,k_b,v_v), & y_1(C,a)&=v_u,\\
    C' &= (k_a,v_v,k_b,v_u), & y_1(C',a)&=v_v.
\end{align*}
After splitting into singleton chunks, both inputs have the same multiset $\{k_a,k_b,v_u,v_v\}$. For the same query, souping therefore produces the same state for $C$ and $C'$. Hence, the model must return the same output on both inputs, though their labels differ, failing the task on at least one input.
\end{proof}

The setting in Proposition \ref{prop:split-pair} is extreme, but underlines the sensitivity of souping to correctly formed chunks. However, this does not mean that souping cannot link information between chunks --- our multi-hop document QA and variable tracking experiments demonstrate that souping models can robustly link information across multiple chunks. In the next section, we observe that souping can readily link disparate chunks when performing the multi-hop generalization of our task.

\subsection{Multi-hop induction}
We consider the multi-hop generalization of our induction-head task. We discuss how souping can continue linking information between chunks to solve the task, though under certain circumstances, it may be less efficient that a sequential input model.

\begin{proposition} \label{prop:multi-hop}
Suppose that every key-value pair is contained in one chunk, each key occurs at most once, and every chunk containing an edge of the queried $n$-hop path receives positive soup weight. Souping and sequential models can solve the $n$-hop induction task by composing the construction in Proposition~\ref{prop:one-hop} across $n$ layers.
\end{proposition}

See Section \ref{sec:proof_multi_hop} for the construction proving this statement. In this way, we establish the basic competency of souping models to solve these classic in-context retrieval tasks. However, one advantage sequential models retain over souping is their ability to learn cross-chunk representations in deeper layers. Because souping encodes chunks independently, it remains impossible to directly encode information between chunks in deeper layers. This expressivity difference may lead to capability separations between sequential and souping models.

For instance, suppose the inputs to our multi-hop induction task are favorably ordered. Consider $C$ node-disjoint chains presented consecutively in forward order, with each chain taking the form
\begin{equation}
    (k_{a_{c,0}},v_{a_{c,1}}),\ (k_{a_{c,1}},v_{a_{c,2}}),\ \ldots,\
    (k_{a_{c,n-1}},v_{a_{c,n}}).
    \label{eq:ordered-chains}
\end{equation}
Suppose the query always names a chain's beginning and asks for its $n$-hop endpoint. Then we observe the following.

\begin{remark} \label{remark:capacity}
    For the ordered, endpoint-only task with $n=2^L$ and $\alpha=1$, an $L+1$-layer sequential model can form a final-layer state containing only the $C$ start-to-end associations, effectively compressing the sequence. Souping models that independently encode single pairs in isolation cannot implement this same compression.
\end{remark}

The construction is given in Section~\ref{sec:ordered_construction}. Each of the first $L$ layers stores reverse associations whose hop length doubles with depth, allowing later tokens to identify earlier nodes in their chain. The final layer then writes
\[
    H_{\mathrm{end}}=\sum_{c=1}^C\vr e(k_{a_{c,0}})\vr e(v_{a_{c,n}})^\top,
\]
so a query at a chain start directly retrieves its endpoint. Independently encoding a single pair cannot perform this cross-chunk preprocessing, though souping can still compose the original relations at query time as in Proposition~\ref{prop:multi-hop}.

A simple counting argument clarifies why the number of retained associations matters for storage.

\begin{proposition} \label{prop:capacity}
A fixed $B$-bit state that must recover $m$ associations in $[N]$ requires $B\geq\lceil m\log_2 N\rceil$, for $N\geq2$.
\end{proposition}

\begin{proof}
There are $N^m$ possible assignments. Distinct assignments must yield distinct states, since they disagree on at least one query. Thus $2^B\geq N^m$. Taking the log of both sides results in the desired bound $B\geq\lceil m\log_2 N\rceil$, for $N\geq2$.
\end{proof}

For the ordered task, the final summary retains $C$ endpoint associations instead of all $Cn$ edges. Hence, Proposition \ref{prop:capacity} suggests that the SSM state for a sequential model could be smaller than that of a souping model to learn a task with a fixed number of chains $C$ and hops $n$. Indeed, to learn $Cn$ associations, a souping model would require $B = Cn \log_2 N$ bits, whereas a sequential model needs a factor of $n$ fewer with just $B = C \log_2 N$ bits.

More generally, sequential encoding may permit useful composed representations in deeper layers. The doubling construction is similar in spirit to the pointer-doubling strategy of \citet{sanford2024k_hop}, and gives a logarithmic-depth sequential solution in this ordered setting. We conjecture that this may yield a depth separation from souping under suitable width and precision constraints.

\subsection{Weighted state souping permits selective retrieval}
While souping may be sensitive to chunk boundaries and exhibit less efficiency for favorably ordered inputs, weighted souping nonetheless retains a decisive advantage over sequential models for long-context retrieval. Poor long-context performance can be understood as a consequence of gating and noisy contexts. In a sequential model with Mamba-like history decay, a contribution at distance $\Delta$ is exponentially decayed by a factor $\alpha^\Delta$. \citet{tong2026gating} have formalized the resulting tradeoff: reducing decay preserves distant signals but permits distractor variance to accumulate, whereas stronger decay controls that variance at the cost of attenuating earlier tokens. Hence, no setting permits a sequential model to retrieve from every token in a long context with even reliability.

Weighted souping bypasses this tradeoff by permitting direct attention on the relevant chunks. In this scheme, distractor variance is determined only by the selected set of chunks rather than the full context size. Earlier chunks are not attenuated compared with later chunks, so weighting is computed evenly across the whole context. Provided that the weights remain concentrated as the number of distractors increases, weighted souping may maintain strong long-context performance when sequential and mean-souped variants fail. In some sense, weighted souping interpolates between recurrent accumulation and attention-like selection. It is not equivalent to full attention (souping weights are scalar, chunk-level, and computed once from the original query). Nonetheless, our experiments support the strong retrieval performance of weighted souping models. Together with their efficiency benefits, weighted state souping presents an attractive paradigm for future models.

\section{Discussion}
Weighted state souping augments reusable context encoding with query-dependent selection. A lightweight routing head combines cached chunk states into a fixed-size recurrent initialization, improving long-context performance across document QA, retrieval, variable tracking, and tool use benchmarks. The results show that a fixed state need not imply uniformly compressing the entire available collection: each query determines which information enters the working state. Inference-time sparsification often improves both accuracy and state-transfer cost, particularly beyond the training pool size. The theory complements these results by demonstrating retrieval and cross-chunk composition, while identifying information lost at chunk boundaries and discussing the potential advantages of sequential processing.

There remain many avenues for further exploration. Our focus in these experiments has been on single-turn interactions, but generalizing souping to multi-turn interactions featuring several rounds of back-and-forth between user and assistant remains an unexplored direction. There are also many opportunities for improving weighted state souping, for instance by replacing fixed-$N$ sparse souping with adaptive sparsity that adjusts to the problem, exploring alternative weighting schemes, or using thinking tokens in conjunction with souping. Our theoretical discussion offers initial intuitions for why state souping works and when it may fail, but there remain many opportunities for investigating potential memory and depth separations between souping and sequential models.

\section*{Acknowledgments}
This work was performed while WLT was at Apple. WLT is supported by a Kempner Graduate Fellowship. CP is supported by an NSF CAREER Award (IIS-2239780), DARPA grants DIAL-FP-038 and AIQ-HR00112520041, the Simons Collaboration on the Physics of Learning and Neural Computation, and the William F. Milton Fund from Harvard University. This work has been made possible in part by a gift from the Chan Zuckerberg Initiative Foundation to establish the Kempner Institute for the Study of Natural and Artificial Intelligence. LLMs were used in implementing the experiments, assisting with proofs, and polishing the writing.

\bibliographystyle{iclr2027_conference}
\bibliography{refs}

\clearpage
\appendix

\section{Additional figures} \label{sec:additional_figs}

\begin{figure}[h]
    \centering
    \includegraphics[width=0.85\linewidth]{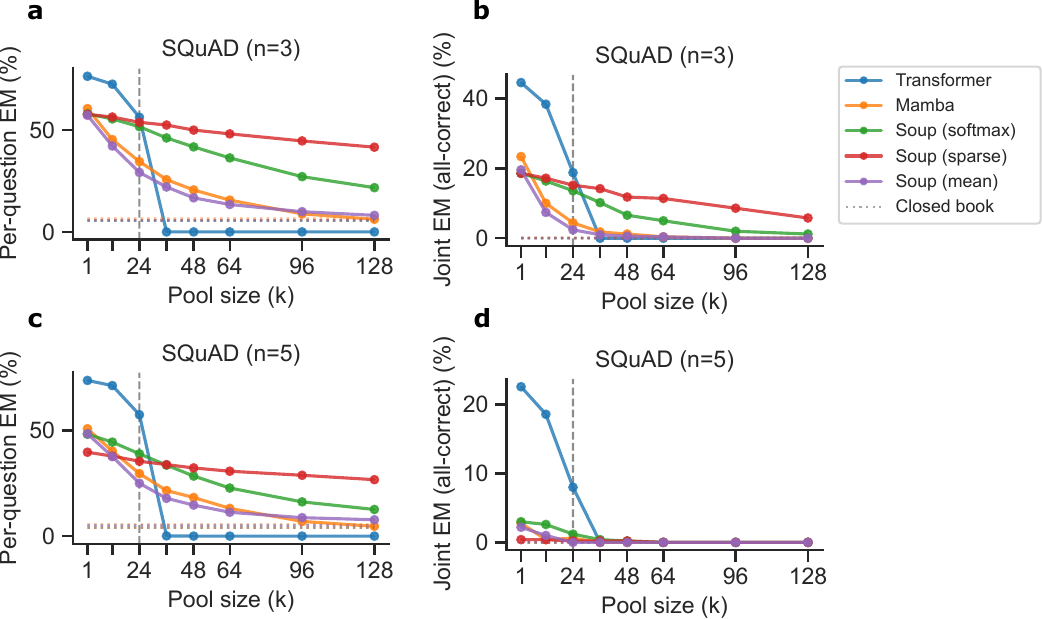}
    \caption{\textbf{Squad-n benchmark results.} As a further test of souping when the task requires attending to multiple documents, we construct a task based on SQuAD called \textit{SQuAD-n}. Each example consists of $n$ individual SQuAD document-question-answer sets. The model is prompted with all $n$ documents (mixed among $\text{max}(k-n,0)$ distractors) followed by the questions, and must produce the answer spans corresponding to each question. We plot both per-question exact match accuracy in (\textbf{a, c}) as well as joint accuracy in (\textbf{b, d}) for $n=3$ and $n=5$. Souping generally performs competitively with vanilla Mamba, and sparse souping in particular maintains long context generalization better than full souping. However, both Mamba and souping underperform Transformers on in-context $K$ in this setting. This may be in part due to the Transformer's sharper ability to track its context, and determine which questions have already been answered and which require an answer still.}
    \label{fig:squad_n_results}
\end{figure}

\begin{figure}
    \centering
    \includegraphics[width=\linewidth]{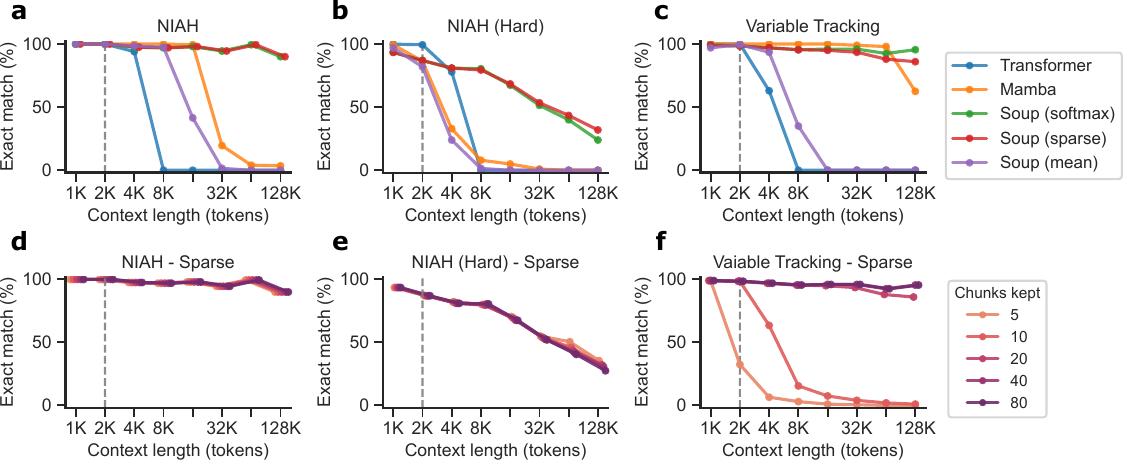}
    \caption{\textbf{Ruler benchmark, additional results.} In the main text, we present results on needle-in-a-haystack \textbf{(a,b)} and variable tracking \textbf{(c)}. We additionally plot a sweep over the number of chunks $N_{top}$ kept by the sparse souping model in \textbf{(d - f)}. For the NIAH tasks, where only a single chunk is relevant, the value of $N_{top}$ is irrelevant. However, for the variable tracking task, relevant variable assignments are distributed across many more chunks, so below a certain threshold, the model fails to learn the task well.}
    \label{fig:ruler_results}
\end{figure}

\begin{figure}
    \centering
    \includegraphics[width=\linewidth]{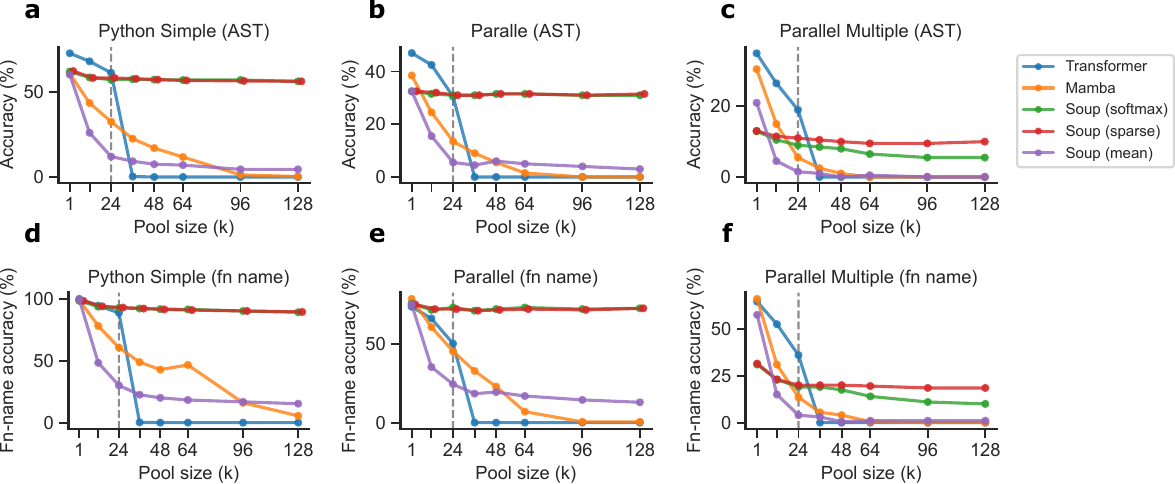}
    \caption{\textbf{BFCL benchmark, additional results.} In the main text, we present results on the Python Simple split of BFCL (\textbf{a}), in which the model must produce a single tool call given a list of tool definitions and a user query. Here, we also plot performance on the Parallel split (the model must call a single tool multiple times) and the Parallel Multiple split (the model must call multiple tools multiple times). We measure accuracy by BFCL's built-in abstract syntax tree (AST)-based checker (\textbf{a - c}) as well as a match based on function name alone (\textbf{d - f}). In general, weighted souping performs well for both Python Simple and Parallel, but a substantial gap opens on small, in-distribution $K$ for Parallel Multiple. Mean souping performs better than weighted souping at $K=1$ (where only gold tool definitions are prompted), indicating that souping may still work well, but the model is struggling to simultaneously learn appropriate weights and use a souped state initialization.}
    \label{fig:bfcl_results}
\end{figure}

\begin{figure}
    \centering
    \includegraphics[width=\linewidth]{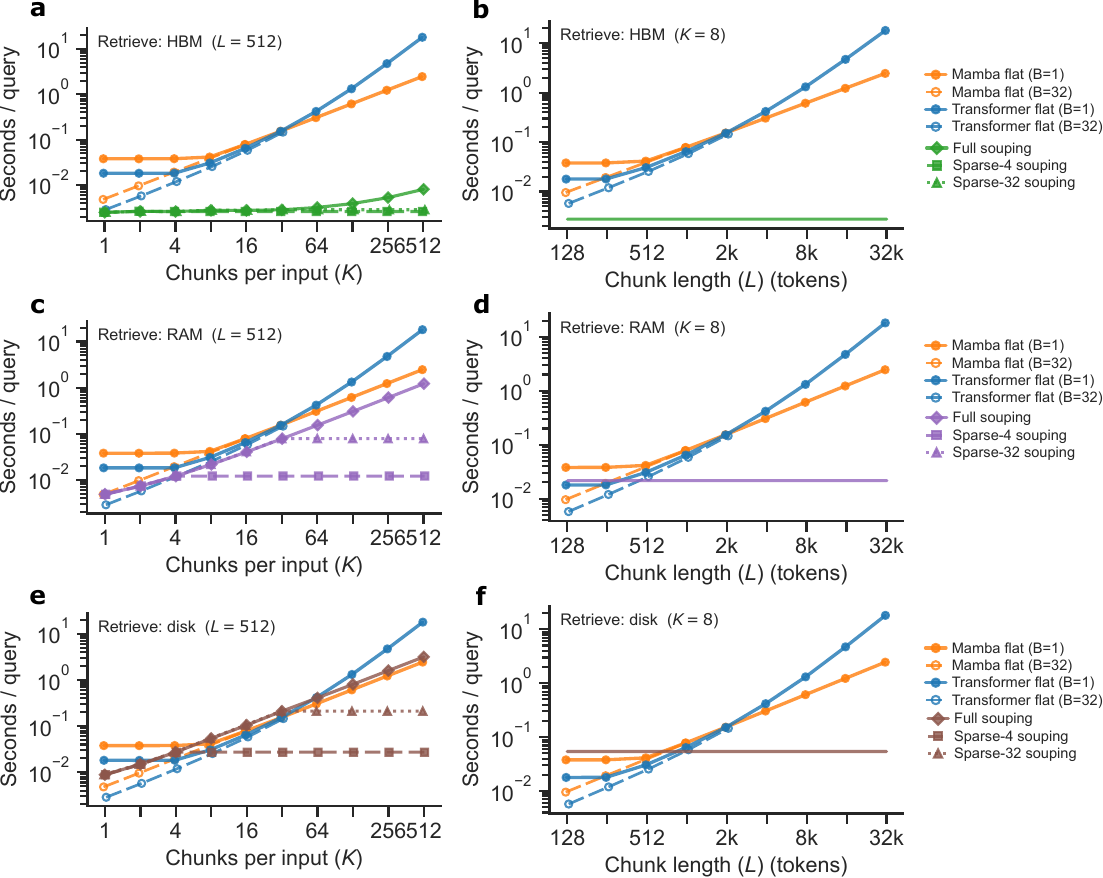}
    \caption{\textbf{Additional measurements on encoding time across models.} We plot additional sets of timings beyond those showin in Figure \ref{fig:compute}. In particular, we show full timings measured for full and sparse souping models retrieving from HBM (\textbf{a, b}), system RAM (\textbf{c, d}), and system disk (\textbf{e, f}). The timings for Transformer and sequential Mamba models in the main text are computed with batch size of 1. To give a sense for what the compute-bound timings look like, we plot the per-example timing for batch size 32 for Transformer and sequential Mamba models (dashed lines).}
    \label{fig:full_compute_time}
\end{figure}

\clearpage
\section{Constructions} \label{sec:constructions}

We use the simplified SSM model in Section~\ref{sec:theory} throughout. In the constructions below, each block matrix column has width $d$, in the order $[\vr e(z_{t-1});\vr u_{t-1}^{(\ell-1)};\vr e(z_t);\vr u_t^{(\ell-1)}]$. At a reset boundary, the previous residual vector is zero. Recurrent states are read after the current token's write, as specified in~\eqref{eq:linear-attention-state}.

\subsection{One-hop retrieval} \label{sec:proof_one_hop}
\begin{proof}
Let $E_K=[\vr e(k_1)\ \cdots\ \vr e(k_N)]$ and $E_V=[\vr e(v_1)\ \cdots\ \vr e(v_N)]$. Orthogonality gives $E_K^\top E_K=E_V^\top E_V=I_N$ and $E_K^\top E_V=0$. Define $A_K=E_KE_K^\top$ and $A_V=E_VE_V^\top$, and use depth $D=1$. The four block columns correspond to $[\vr e(z_{t-1});\vr u_{t-1}^{(0)};\vr e(z_t);\vr u_t^{(0)}]$. Choose
\begin{equation}
    \begin{split}
    W_k^{(1)}&=[A_K\ \ 0\ \ 0\ \ 0],\qquad W_v^{(1)}=[0\ \ 0\ \ A_V\ \ 0],\\
    W_q^{(1)}&=[0\ \ 0\ \ A_K\ \ 0],\qquad U=E_V^\top.
    \end{split}
    \label{eq:one-hop-weights}
\end{equation}
Together with the stated embeddings, convolution, residual map $W_o$, zero initial state, zero biases, and any fixed $\alpha\in(0,1]$, these choices specify all model parameters. We suppress the layer superscript below.

At a value token $v_{f(a)}$, the previous token is $k_a$, so the state increment is $\vr e(k_a)\vr e(v_{f(a)})^\top$. At a key token, including the query, $W_v\vr x_t=\vr0$, so there is no write. At the query the state is therefore
\[
    H_q=\sum_{a\in\operatorname{dom}(f)}\gamma_a\vr e(k_a)\vr e(v_{f(a)})^\top.
\]
For sequential encoding, $\gamma_a=\alpha^{T-t_a}$, where $t_a$ is the position of the value token and $T$ is the query position. For souping, if this pair lies in chunk $i(a)$ and its value has local position $t_{a,i(a)}$, then $\gamma_a=w_{i(a)}\alpha^{m_{i(a)}-t_{a,i(a)}+1}$. These coefficients include the decay at the query token. Orthogonality now gives $\vr h_q=\gamma_q\vr e(v_{f(q)})$. The residual update gives $\vr u_q^{(1)}=\vr h_q$ because $\vr u_q^{(0)}=0$, so $U\vr u_q^{(1)}$ has its only nonzero entry, $\gamma_q>0$, at coordinate $f(q)$. Argmax returns the correct value. In the sequential model, continuous convolution across pair boundaries gives the same writes, since keys have zero value projection.
\end{proof}

\subsection{Multi-hop retrieval with additive residuals} \label{sec:proof_multi_hop}
\begin{proof}

Use the embedding matrices and $A_K,A_V$ from the one-hop proof, and let $F=E_KE_V^\top$, so $F\vr e(v_a)=\vr e(k_a)$. The four block columns below correspond to $[\vr e(z_{t-1});\vr u_{t-1}^{(\ell-1)};\vr e(z_t);\vr u_t^{(\ell-1)}]$. In the first layer, use
\[
\begin{split}
    W_k^{(1)}&=[A_K\ \ 0\ \ 0\ \ 0],\qquad
    W_v^{(1)}=[0\ \ 0\ \ A_V\ \ 0],\\
    W_q^{(1)}&=[0\ \ 0\ \ A_K\ \ 0].
\end{split}
\]
Since $\vr u_t^{(0)}=\vr0$, this is the one-hop construction. Every context output is zero: a context key has not yet been written, and a value has zero query projection. At the final query, $\vr u_q^{(1)}=\gamma_q\vr e(v_{f(q)})$.

Subsequent layers must add the next retrieval and subtract the previous working vector. This can be implemented by the attention block itself. Choose a fixed unit vector $\vr b=\vr e(v_1)$, orthogonal to all key embeddings, and let $\vr c=\sum_{a=1}^N\vr e(k_a)$ and $B=\vr b\vr c^\top$. Thus $B\vr e(k_a)=\vr b$ and $B\vr e(v_a)=\vr0$. For $\ell\geq2$, set
\[
\begin{split}
    W_k^{(\ell)}&=[A_K\ \ 0\ \ B\ \ 0],\qquad
    W_v^{(\ell)}=[0\ \ 0\ \ A_V\ \ {-I_d}],\\
    W_q^{(\ell)}&=[0\ \ 0\ \ B\ \ F].
\end{split}
\]
At context positions the working vectors remain zero: a value writes the original association and has zero query projection; a key has zero value projection and queries $\vr b$, orthogonal to all association keys. Consequently, every layer's context state contains precisely the original associations, for both sequential encoding and independently encoded chunks.

At the final query, write $\vr u=\vr u_q^{(\ell-1)}$ and let
\[
    M=\sum_{a\in\operatorname{dom}(f)}\gamma_a\vr e(k_a)\vr e(v_{f(a)})^\top
\]
be the context contribution after multiplying the initialized context state by $\alpha$ when processing the query token, with the same coefficients as in the one-hop proof. The query's key, value, and query projections are respectively $\vr b$, $-\vr u$, and $F\vr u+\vr b$. Since the recurrence reads after the current write,
\[
    \vr h_q^{(\ell)}=(M-\vr b\vr u^\top)^\top(F\vr u+\vr b)
      =M^\top F\vr u-\vr u.
\]
Here $M^\top\vr b=\vr0$, $\vr b^\top F\vr u=0$, and $\|\vr b\|=1$. The additive residual therefore gives $\vr u_q^{(\ell)}=M^\top F\vr u_q^{(\ell-1)}$. Inductively,
\[
    \vr u_q^{(\ell)}=
      \left(\prod_{j=0}^{\ell-1}\gamma_{f^j(q)}\right)\vr e(v_{f^\ell(q)}).
\]
Use $U=E_V^\top$ and the same $\alpha$ at every layer. Every factor is positive, so the final readout identifies $v_{f^n(q)}$.
\end{proof}

\subsection{Ordered-chain construction} \label{sec:ordered_construction}

We give the explicit construction for Remark~\ref{remark:capacity}. Let $n=2^L$, and let the $C$ chains in~\eqref{eq:ordered-chains} have mutually disjoint node sets, so $C(n+1)\leq N$. Set $\alpha=1$ and use the orthogonal embeddings, residual path, and convolution from Section~\ref{sec:theory}. Convolution continues across pair boundaries in the sequential model. All recurrent states start at zero and all biases are zero. Write $A_K=E_KE_K^\top$, $A_V=E_VE_V^\top$, and $F=E_KE_V^\top$, as in the retrieval proofs.

The first $L$ layers store reverse associations. Using the same block columns as in the proof of Proposition~\ref{prop:multi-hop}, set
\[
\begin{split}
    W_k^{(1)}&=[0\ \ 0\ \ F\ \ 0],\qquad
    W_v^{(1)}=[A_K\ \ 0\ \ 0\ \ 0],\\
    W_q^{(1)}&=[0\ \ 0\ \ A_K\ \ 0].
\end{split}
\]
For $2\leq\ell\leq L$, let $G=F^\top$ and set
\begin{equation}
\begin{split}
    W_k^{(\ell)}&=[0\ \ 0\ \ F+G\ \ 0],\qquad
    W_v^{(\ell)}=[0\ \ I_d\ \ 0\ \ {-I_d}],\\
    W_q^{(\ell)}&=[0\ \ 0\ \ G\ \ I_d].
\end{split}
\label{eq:reverse-weights}
\end{equation}
Working vectors at value tokens remain zero, whereas those at key tokens lie in the key-embedding span. In the first layer, each value token writes its reverse one-hop association using the preceding key's original embedding. For $\ell\geq2$, a value token $v_{a_{c,j}}$ writes the reverse association with key $\vr e(k_{a_{c,j}})$ and value equal to the preceding key's working vector. A key token $k_a$ instead writes $-\vr e(v_a)(\vr u_t^{(\ell-1)})^\top$ and queries $\vr u_t^{(\ell-1)}+\vr e(v_a)$. The value-embedding direction $\vr e(v_a)$ is orthogonal to all reverse-association keys and to the cancellation keys at earlier context positions, since each context key is unique. Its current write therefore subtracts exactly $\vr u_t^{(\ell-1)}$ from the attention output. After residual addition, the working vector contains only the newly retrieved reverse jump. These fixed key and value spans suffice at every layer.

Index edges by $j=1,\ldots,n$, so edge $j$ is $(k_{a_{c,j-1}},v_{a_{c,j}})$. Let $\vr u_{c,j}^{(\ell)}$ denote the working vector at the key token of edge $j$ in chain $c$, for $\ell\geq1$. We claim
\begin{equation}
    \vr u_{c,j}^{(\ell)}=
    \begin{cases}
      \vr e(k_{a_{c,j-2^\ell}}),& j\geq2^\ell,\\
      \vr0,& j<2^\ell.
    \end{cases}
    \label{eq:ordered-invariant}
\end{equation}
For the base case $\ell=1$, the key token of edge $j$ is $k_{a_{c,j-1}}$. For $j\geq2$, the preceding edge has already written the reverse association from $k_{a_{c,j-1}}$ to $k_{a_{c,j-2}}$; for $j=1$ no such association exists. This proves the stated base case. For $\ell\geq2$, suppose the claim holds for layer $\ell-1$. Put $s=2^{\ell-1}$. At the value token of edge $j$, layer $\ell$ writes
\[
    \begin{cases}
      \vr e(k_{a_{c,j}})\vr e(k_{a_{c,j-s}})^\top,&j\geq s,\\
      0,&j<s.
    \end{cases}
\]
The component of the state addressed by key embeddings therefore stores reverse jumps of length $s$. At the key token of edge $j$, the corresponding query component is $\vr e(k_{a_{c,j-s}})$ when $j\geq s$. If $j\geq2s$, edge $j-s$ has already been processed, so its reverse association returns $\vr e(k_{a_{c,j-2s}})$. For $s\leq j<2s$, no matching association has been written; for $j<s$, this query component is zero. Both cases give zero retrieval. In layers $\ell\geq2$, the additional query component cancels the incoming working vector as described above. Orthogonality and disjoint chain nodes exclude all other matches, proving~\eqref{eq:ordered-invariant} after residual addition.

For $L\geq1$, the final layer writes forward endpoint associations. Set
\[
\begin{split}
    W_k^{(L+1)}&=[0\ \ A_K\ \ 0\ \ 0],\qquad
    W_v^{(L+1)}=[0\ \ 0\ \ A_V\ \ 0],\\
    W_q^{(L+1)}&=[0\ \ 0\ \ A_K\ \ 0],\qquad U=E_V^\top.
\end{split}
\]
By~\eqref{eq:ordered-invariant}, the preceding key's working vector is nonzero only for $j\geq2^L=n$. Thus only the final edge of each chain writes to this layer, producing
\[
    H_{\mathrm{end}}=\sum_{c=1}^C\vr e(k_{a_{c,0}})\vr e(v_{a_{c,n}})^\top.
\]
At a chain-start query, every reverse layer returns zero: no reverse association points back from a chain start, and the earlier occurrence of this start wrote no cancellation term. Hence the query enters the final layer with zero working vector, writes zero, and uses its original key embedding to retrieve the appropriate endpoint. The readout $U$ identifies that value. For example, four-hop retrieval uses two reverse layers, storing jumps of lengths one and two, followed by a final layer associating each chain start with its four-hop endpoint.

\clearpage
\section{Implementation details} \label{sec:impl_details}

\subsection{Tasks} \label{sec:tasks_appendix}

Across the document-QA and tool-use benchmarks we probe long-context robustness with a shared \emph{pooling} construction. Each example carries one or more \emph{gold} items --- the documents or tool definitions required to answer the query --- which we always retain; we then add sampled \emph{distractors} until the pool holds $k$ items in total and shuffle so the gold position is randomized. Thus $k$ counts the \emph{total} number of in-context chunks (not the number of distractors), and $k=1$ denotes a gold-only pool. Distractors are drawn without replacement from a split-local pool that excludes the example's own gold items, so training and evaluation distractors never cross split boundaries. During finetuning the pool size is sampled per example as $k \sim \mathrm{Uniform}[1, k_{\max}]$; at evaluation we sweep $k$ well beyond $k_{\max}$ to test generalization. All tasks share the Nemotron-Nano tokenizer \citep{nvidia2025nemotronnano2} (vocabulary size $131{,}072$) and are evaluated on the public validation split (the official test splits are hidden). RULER is handled differently, as described below.

\paragraph{Berkeley Function Calling Leaderboard (BFCL).} A standard tool-use benchmark \citep{patil2025bfcl}: the model is prompted with a list of tool definitions and must emit a tool call that solves the user query. Before evaluation we finetune each model for tool calling on a filtered subset of NVIDIA's Nemotron agentic SFT data \citep{nvidia2025nemotronnano2}, keeping the ${\approx}228{,}700$ examples whose first assistant turn is a tool call. We evaluate on BFCL, reporting the \emph{simple python} (one function, one call; $400$ examples), \emph{parallel} (one function, multiple calls; $200$ examples), and \emph{parallel multiple} (multiple functions, multiple calls; $200$ examples) splits. Tool pools are built with a global BM25 retriever \citep{robertson2009bm25} over the ${\approx}14{,}700$-tool non-live vocabulary: candidate tools are scored by BM25, truncated to the top $k$, and the gold tool(s) are then force-included and randomly positioned. Training samples $k \sim \mathrm{Uniform}[1,24]$; the headline evaluation sweeps $k \in \{1,12,24,36,48,64,96,128\}$. Correctness is scored by BFCL's abstract-syntax-tree checker, which validates the function name and every argument value against the reference-answer set. We additionally report an order-independent function-name-only match.

\paragraph{Stanford Question Answering Dataset (SQuAD).} A single-document extractive QA dataset \citep{rajpurkar2016squad}: given a Wikipedia paragraph, the model quotes a text span answering the query. We use SQuAD~v1.1 and draw distractor paragraphs from the same split. Training uses $k \sim \mathrm{Uniform}[1,24]$; evaluation sweeps $k \in \{1,12,24,36,48,64,96,128\}$ over a fixed subset of $500$ validation questions (shared across all arms via a common seed). We report exact match (EM) accuracy under SQuAD normalization (lowercasing, and removal of punctuation, articles, and extra whitespace). Dotted lines in the figures report closed-book accuracy ($k=0$, no documents in context) as a memorization floor.

\paragraph{HotpotQA.} A multi-hop QA dataset \citep{yang2018hotpotqa}. Each context paragraph becomes one document, and the annotated supporting paragraphs are marked gold (a multi-gold example). Pooling, evaluation $k$-sweep, example count, and metrics match SQuAD.

\paragraph{MuSiQue.} A multi-hop QA dataset engineered to require genuine reasoning over many hops and to resist single-document shortcuts \citep{trivedi2022musique}; it is a harder version of HotpotQA. We use the answerable variant (\texttt{musique\_ans}), treating the annotated supporting paragraphs as gold and scoring against the gold answer and its aliases. Pooling, evaluation $k$-sweep, example count, and metrics match SQuAD.

\paragraph{SQuAD-$n$.} To further stress multi-document attention, we construct \emph{SQuAD-$n$}: we merge $n$ SQuAD examples with distinct gold paragraphs into a single example whose prompt contains all $n$ gold paragraphs (mixed among $\max(k-n,0)$ distractors) followed by the $n$ questions, labeled $(1)\ldots(n)$; the model must return the $n$ answer spans in order. We use $n \in \{3,5\}$. We report micro-averaged EM over all (example, sub-question) pairs, and a strict \emph{joint} EM that credits an example only when every sub-answer is an exact match.

\paragraph{RULER.} A synthetic long-context benchmark \citep{hsieh2024ruler}. We use an in-repository reimplementation of its tasks and focus on three subtasks: single needle-in-a-haystack (a magic-number ``needle'' line hidden among Paul Graham essays), multi-key NIAH (the needle is hidden among confusable distractor needles), and variable tracking (chains of variable assignments embedded in filler text, where the model resolves the value of a queried variable by following a four-hop chain). Unlike the QA and tool-use tasks, RULER has no natural chunks and no distractor pool; following \citet{jafari2025surprising_soup} we chunk each input into consecutive $256$-token blocks, so the ``pool size'' is $\lceil \text{context length}/256 \rceil$ and grows with the context. A block is treated as gold if its token span overlaps a needle sentence (NIAH) or an assignment in the queried chain (variable tracking). We cap the training context at $2048$ tokens ($8$ chunks) and evaluate at context lengths $\{1,2,4,8,16,32,64,128\}\text{K}$ tokens over $200$ examples per (subtask, length).

\subsection{Models} \label{sec:models}

\paragraph{Architectures.} Our SSM is a $1.5$B-parameter Mamba-2 model \citep{dao2024mamba2}: $32$ Mamba-2 mixer layers (with no interleaved attention or MLP blocks), hidden size $2048$, RMSNorm, and no positional embedding. Each mixer uses a state-expansion factor of $2$ (inner dimension $4096$) partitioned into $64$ heads of head dimension $64$, state dimension $128$, $8$ groups, and a convolution kernel of width $4$. The recurrent state is therefore $4096 \times 128 = 524{,}288$ elements per layer, or ${\approx}16.8$M elements ($33.6$\,MB in bf16) across all $32$ layers --- this is the size of a single cached state encoding. Our Transformer is a $1.5$B-parameter LLaMA-style model: $20$ layers, hidden size $2048$, SwiGLU feed-forward width $5504$, $16$ attention heads with full multi-head attention, rotary position embeddings, and RMSNorm. The two models are matched in parameter count (${\approx}1.5$B). Both are pretrained on the Nemotron-CC corpus \citep{su2025nemotroncc} (on the order of $300$B tokens) with the same optimizer schedule (AdamW, peak learning rate $3\times10^{-4}$, cosine decay), share the Nemotron-Nano tokenizer, and run in bf16. We deliberately give the Transformer rotary embeddings so that it can process contexts longer than those seen in training, so that the sequential-Transformer baseline is given a fair chance at length generalization.

\paragraph{Weighted-souping read-out.} We augment the SSM with a lightweight read-out head that produces the key and query vectors of Section~\ref{sec:experiments}. Concretely, we read the last-token \emph{residual} activation at the final decoder block and pass it through two learned linear projections $W_k, W_q \in \R^{d \times d'}$ with $d=2048$ and $d'=128$ (equivalently, a single projection to $\R^{2d'}$ split into a key half and a query half). The two projections are separate parameters, initialized to identical values ($W_k = W_q$) so the head begins as the parameter-free cosine ranker over the base model's activations; they are untied thereafter and trained independently. The head is the only component initialized from scratch; it is trained jointly with the base model but placed in its own optimizer group at learning rate $\text{base LR} \times \texttt{lr\_scale}$, with $\texttt{lr\_scale}=10$.

\paragraph{Cosine normalization.} The weight assigned to chunk $i$ is $\sigma(\text{score}_i)$, where the score is a dot-product between the chunk key $\vr{k}_i$ and the query vector $\vr{u}$. Rather than the usual $1/\sqrt{d}$ dot-product scaling, we $\ell_2$-normalize the key and query and scale their cosine similarity by a learnable inverse temperature $\gamma$: $\text{score}_i = \gamma \, \frac{\vr{k}_i \cdot \vr{u}}{\|\vr{k}_i\|\,\|\vr{u}\|}$. We initialize $\gamma=14$ and cap it at $100$. Because the pre-nonlinearity logits are bounded to $[-\gamma, \gamma]$ regardless of pool size, the weighting stays well-conditioned as $k$ grows, which we found important for generalizing to pool sizes far beyond the training distribution. Our headline weighting function $\sigma$ is the softmax.

\paragraph{Sparse souping.} Sparse souping keeps only the top-$N$ weights and zeroes the rest; with softmax we apply the mask before the nonlinearity so the surviving weights still sum to one. It is applied at inference time and reuses the softmax checkpoints unchanged. We use $N_{top}=4$ for the document-QA and BFCL figures and $N_{top}=20$ for RULER; the larger $N_{top}$ is needed for variable tracking, where gold information is spread across many chunks (we sweep $N_{top}$ in Figure~\ref{fig:ruler_results}).

\paragraph{Auxiliary routing loss.} We optionally supervise the weights toward the known gold chunks with an auxiliary loss, $\mathcal{L} = \mathcal{L}_{\text{SFT}} + \lambda\,\mathcal{L}_{\text{aux}}$, where $\mathcal{L}_{\text{aux}}$ is a cross-entropy over the $k$ chunks under softmax weighting, computed from the same scores used to combine the states. We set $\lambda=1$ for RULER and tool use, where the gold chunks are reliably identifiable, and $\lambda=0$ for document QA, where the head learns the weighting purely from the downstream SFT signal.

\paragraph{Souping finetuning.} Finetuning is necessary to learn the souping format \citep{jafari2025surprising_soup}. All souping models are finetuned from the pretrained base with the softmax weighting, residual read-out, and cosine normalization above. For document QA we train for $3$ epochs (SQuAD, HotpotQA) or $5$ epochs (MuSiQue) at base learning rate $5\times10^{-6}$ with no auxiliary loss, from the train splits included with the benchmarks. For RULER we train for $6000$ steps at base learning rate $10^{-5}$ with $\lambda=1$, chunked at $256$ tokens. For tool use we \emph{warm-start} from a tool-calling-SFT checkpoint already finetuned on the Nemotron Agentic SFT split \citep{nvidia2025nemotronnano2}, and train for an additional one epoch at base learning rate $10^{-5}$ with $\lambda=1$ to learn the souping format. In every case the read-out head uses the $10\times$ learning-rate multiplier described above. The sequential baselines are finetuned on the concatenated input for matched steps at learning rate $10^{-5}$ (QA/RULER) or $5\times10^{-5}$ (tool use), so that all models see matched data and compute.

\paragraph{Decoding and evaluation.} All models decode greedily. For tool use, decoding is constrained by a tokenizer-level logit mask that enforces the tool-call envelope (\texttt{<TOOLCALL>[\,\ldots\,]</TOOLCALL>}); the constraint is applied identically to the souping arms and the sequential baselines. The QA and RULER tasks decode without constraints. Evaluation pool sizes, context lengths, and example counts are as given in Appendix~\ref{sec:tasks_appendix}.

\subsection{Computational benchmark} \label{sec:compute_appendix}

The timings in Figure~\ref{fig:compute} are measured on a single AWS \texttt{p5.48xlarge} node ($8\times$ H100). We compare four inference paths: a sequential Transformer, a sequential Mamba, full weighted souping, and sparse souping. For the souping paths, we time the encoding of a $128$-token query and the computation of the softmax weights, while the chunk state encodings themselves are \emph{retrieved} rather than recomputed, from one of three memory tiers: GPU high-bandwidth memory (HBM), system RAM, or system disk. Each cached Mamba-2 encoding is the $33.6$\,MB bf16 state described above, so retrieval cost is dominated by the bandwidth of the tier in which the encodings live; retrieval from HBM is far faster than from RAM, which in turn is faster than from disk, and in the worst case (short chunks retrieved from disk) retrieval can be slower than re-encoding.

\end{document}